\documentclass{article}

\usepackage{microtype}
\usepackage{graphicx}
\usepackage{subfigure}
\usepackage{booktabs} 
\usepackage{subcaption}

\usepackage{hyperref}
\usepackage{caption}

\usepackage[accepted]{icml2024}

\usepackage{amsmath}
\usepackage{amssymb}
\usepackage{mathtools}
\usepackage{amsthm}
\usepackage{dsfont}

\usepackage[capitalize,noabbrev]{cleveref}

\theoremstyle{plain}
\newtheorem{theorem}{Theorem}[section]

\theoremstyle{definition}
\newtheorem{definition}[theorem]{Definition}

\theoremstyle{remark}

\theoremstyle{observation}
\newtheorem{observation}[theorem]{Observation}

\usepackage[textsize=tiny]{todonotes}

\begin{document}

\twocolumn[
\icmltitle{Marginal Laplacian Score for Feature Selection}



\icmlsetsymbol{equal}{*}

\begin{icmlauthorlist}
\icmlauthor{Guy Hay}{equal,to}
\icmlauthor{Ohad Volk}{equal,to}
\end{icmlauthorlist}

\icmlcorrespondingauthor{Guy Hay}{guy.hay@intel.com}
\icmlaffiliation{to}{AI Solution Group, Intel Corporation}

\icmlkeywords{Machine Learning, Feature Selection, Imbalance, Unsupervised}

\vskip 0.3in
]



\printAffiliationsAndNotice{\icmlEqualContribution} 

\begin{abstract}
High-dimensional imbalanced data poses a machine learning challenge. In the absence of sufficient or high-quality labels, unsupervised feature selection methods are crucial for the success of subsequent algorithms. Therefore, we introduce a Marginal Laplacian Score (MLS), a modification of the well known Laplacian Score (LS) tailored to better address imbalanced data. We introduce an assumption that the minority class or anomalous appear more frequently in the margin of the features. Consequently, MLS aims to preserve the local structure of the dataset's margin. We propose its integration into modern feature selection methods that utilize the Laplacian score. We integrate the MLS algorithm into the Differentiable Unsupervised Feature Selection (DUFS), resulting in DUFS-MLS. The proposed methods demonstrate robust and improved performance on synthetic and public datasets.
\end{abstract}

\section{Introduction}
\label{Introduction}
High-dimensional imbalanced data presents a challenge in machine learning \cite{10.1145/3343440,ali2019imbalance,hasanin2019severely}. In recent years, the prevalence of high-dimensional datasets has grown, even in domains where accurate labeling of samples is challenging. This poses a significant challenge for machine learning algorithms: the task of discerning and disregarding potentially noisy or irrelevant dimensions in the absence of sufficient or any labeling \cite{9229137}. Supervised classification algorithms such as Lasso Regression, Random Forest, and Gradient Boosted Decision Trees have demonstrated the capacity to disregard irrelevant dimensions \cite{muthukrishnan2016lasso, hasan2016feature, rao2019feature}. This capability is a shared trait among various supervised feature selection algorithms, proving beneficial in enhancing the performance of classification endeavors \cite{kornbrot2014point, robnik2003theoretical}. The pairing of these two methods typically generates strong results but there are three common cases where they have been known to fail: dimensionality greater than the samples, highly imbalanced data, and absence of any labeling.

These scenarios pose a significant challenge as supervision relies on a sufficient number of samples, which is not fulfilled in these instances. While recent research has addressed highly imbalanced supervised tasks \cite{maldonado2018dealing, yan2019parameter, VOLK2024200316}, in highly imbalanced datasets for unsupervised tasks, there is a need for the entire process of feature selection and prediction to be unsupervised, using unsupervised classification algorithms as well \cite{breunig2000lof, liu2008isolation, hay2023sortad}. One of the most notable unsupervised feature selection methods is the Laplacian Score (LS) \cite{he2005laplacian}. More advanced techniques in recent years has led to significant developments, such as the introduction of Differentiable Unsupervised Feature Selection (DUFS) \cite{lindenbaum2021differentiable}, which is a modern enhancement of the Laplacian Score method, focused on ignoring irrelevant random features.

The unsupervised feature selection method focuses on identifying features that effectively capture the structure and information across the entire dataset. In highly imbalanced datasets, these features might prove irrelevant for downstream tasks, such as binary classification or anomaly detection. Since the most informative features are those that best separates the minority class from the majority class, and not necessarily captures all the information the dataset contains. As the majority class predominantly shapes the data's structure, unsupervised feature selection methods prioritizing the identification of features without accounting for imbalanced data tend to capture the underlying structure of the majority class, neglecting the minority class. These feature may be useful for distinguishing changes in the majority class but often fall short in capturing the desired differences between the minority class to the majority class. Therefore, there is a need for unsupervised feature selection methods better suited for imbalanced data.

With this objective, we propose the Marginal Laplacian Score (MLS). MLS is an improvement to LS, that focuses on capturing the marginal information of the dataset. We illustrate a theoretical two marginal features in the Appendix \ref{app:marginal_feature_illustration}. The emphasis on marginal information stems from our assumption that minority classes or anomalies typically exhibit abnormal behaviors, often found at the edges of a feature's distribution. For example, many engineering tests are designed so defects will hold values on the margins of the samples distribution \cite{latif2011idvp}. Furthermore, MLS, as a modification of LS, allows for the adaptation of algorithms that leverage LS, such as DUFS. To showcase this ability we introduce DUFS-MLS.

To illustrate our proposed method's effectiveness, we apply it to various unsupervised machine learning tasks. First, we consider synthetic data that includes informative marginal features along with irrelevant features. Secondly, we evaluate our method on 14 imbalanced public datasets from a wide range of domains, under two setups: unmodified and noisy. The unmodified setup represents the dataset as is, while the noisy setup includes additional noisy features. Across all setups, MLS and DUFS-MLS consistently demonstrate improved results. Furthermore, we empirically demonstrate the validity of our assumption in public datasets spanning multiple domains.

We summarize our contribution as follows:
\begin{itemize}
    \item We propose Marginal Laplacian Score, a modification of Laplacian Score better suited for imbalanced data. Also, incorporating it into DUFS, creating DUFS-MLS.
    \item We empirically show the correctness of our assumption that the minority class appears more frequently in the margin of the features. An assumption that helps focus unsupervised feature selection algorithms for imbalance data.
    \item We evaluate the proposed MLS and DUFS-MLS algorithms on synthetic and public datasets.
\end{itemize}

\section{Previous Methods}

The Laplacian Score (LS) \cite{he2005laplacian} is based on Laplacian Eigenmaps \cite{belkin2001laplacian} and Locality Preserving Projection \cite{he2003locality}, and attempts to assign a score that  quantifies the extent to which a particular feature preserves the local distances between samples. This is established by requiring a correlation between the distances among feature samples and the overall sample distances. Consequently, distant samples will indicate distant feature values and vice versa. LS's equation is given in Eq. \ref{eq_laplacian_score}.

\begin{equation}
L_r = \sum_{i,j}\frac{(f_{r_i}-f_{r_j})^2S_{ij}}{Var(f_r)}=\frac{\Tilde{f}_r^TL\Tilde{f}_r}{\Tilde{f}_r^TD\Tilde{f}_r}
\label{eq_laplacian_score}
\end{equation}
Where $f_r$ is the $r$th feature, $S_{ij}=e^{\frac{\| x_i - x_j \|}{t}}$ is a weighting factor, $\Tilde{f}_r=f_r-\frac{f_r^TD\mathds{1}}{\mathds{1}^TD\mathds{1}}$, $L=D-S$, $D=diag(S\mathds{1})$. Intuitively, $S_{ij}$ captures the global distance and  $(f_{r_i}-f_{r_j})^2$ the local distance. A low LS's score, indicating the feature preserves the distances well, is when samples that are globally close, meaning that  $S_{ij}$ has a large value, have a low $(f_{r_i}-f_{r_j})^2$, and vise versa for farther away samples globally. The sum in Eq. \ref{eq_laplacian_score} are pairs for every two sample combination, in imbalance datasets the majority of pairs will strictly be of samples from the majority class. Thus, leading to the predominant of the majority classes distribution on the underlying data structure being learned. In addition, when considering the entire dataset, correlated features will dominate the global distances, causing both terms in the numerator to align. Consequently, this alignment results in a low score, leading to the selection of these features.

As mentioned in the introduction, Differentiable Unsupervised Feature Selection (DUFS) is an extension of LS designed to disregard irrelevant random features. The parametric free equation is given in Eq. \ref{eq:dufs}. DUFS introduces stochastic gates to the input, enabling the calculation of LS for various feature sets, thereby preventing LS from assigning low scores to irrelevant features. The authors argue that in scenarios where numerous irrelevant features exist, they can form structures smaller than the global structure but still substantial enough to influence LS. Hence, computing LS using different masks is more likely to neutralize the smaller structures rather than the correct, larger structure, enabling the stochastic gate to effectively suppress the irrelevant features.

\begin{equation}
    L(\mu)=-\frac{\mathrm{Tr}[\Tilde{F}^TL_{\Tilde{F}}\Tilde{F}]}{m\sum_{i=1}^d\mathds{P}(Z_i\geq0)+\delta}
    \label{eq:dufs}
\end{equation}

Where $F$ is the whole feature space, $\Tilde{F}=F\odot Z$, $\odot$ is an element-wise multiplication, the Hadamard product, $L_{\Tilde{F}}$ be the random walk graph Laplacian computed on $\Tilde{F}$ and $Z$ is the stochastic gate, $m$ is the number of gates and $d$ is the number of features, and $\delta$ is a small constant.

\section{Method}
\subsection{Preliminaries}

We define $\mathcal{X}$ as the original sample space, $x_i \in \mathcal{X}$ as sample $i$, and $\mathcal{F}$ as the feature space. It is important to note that not all feature margins are informative. If a feature is strongly skewed right, we expect most anomalies to lay on the right side, and vise versa. Therefore, we would define the margin for right skewed features to be only the right side of the feature, neglecting the left margin. Thus, each feature's skewness will be used to determine which side is of interest. The following definitions will help us define for each feature the relevant margin and to define indicator functions for said margin. We will further generalize the feature's margin definition to define the dataset margin.

\begin{definition}[Feature Skewness Set]
Given skewness thresholds $s_{right},s_{left},s_{two-side}\in \mathbb{R}$, and features $\mathcal{F}$. A feature skewness set is the set of all features that their skewness values are in a given range specified by the given thresholds. There are three sets of interest defined as follows:
\begin{equation}
    \mathcal{F}^{right}=\{f \in \mathcal{F} | S(f) >= s_{right}\}
    \nonumber
\end{equation}
\begin{equation}
    \mathcal{F}^{two-side}=\{f \in \mathcal{F} | s_{left} < S(f) < s_{right}\}
    \nonumber
\end{equation}
\begin{equation}
    \mathcal{F}^{left}=\{f \in \mathcal{F} | S(f) <= s_{left}\}
    \nonumber
\end{equation}
\end{definition}
Where $S$ is a function that calculates the skewness of the feature. 
\begin{definition}[Feature Margin]
Given quantile $q$ of feature $f_r$, the feature's margin is defined as the set of all samples in a specific range dictated by $q$. There are three sets of interest defined as follows:
\begin{equation}
    \mathcal{M}^{right}_r=\{x_i\in \mathcal{X}| x_i > q\}
    \nonumber
\end{equation}
\begin{equation}
    \mathcal{M}^{two-sided}_r=\{x_i\in \mathcal{X}| x_i < (1-\frac{q}{2}) \vee x_i > \frac{q}{2} \}
    \nonumber
\end{equation}
\begin{equation}
    \mathcal{M}^{left}_r=\{x_i\in \mathcal{X}| x_i < (1-q)\}
    \nonumber
\end{equation}
\end{definition}

\begin{definition}[Margin of Interest]
Given a feature $f\in \mathcal{F}^{s}$ where $s\in \{right, two-sided, left \}$. The margin of interest is $f_r$'s feature margin corresponding to the specified $s$, $\mathcal{M}^{s}_r$. The margin of interest for feature $f_r$ is denoted by $\mathcal{M}_r$. 
\end{definition}

\begin{definition}[Dataset Margin]
Given $k\in \mathbb{N}$, the dataset margin is the set of all samples that are in at least $k$ margin's of interest. 
\begin{equation}
    \mathcal{M}^{k} = \{ x \,|\, |\{r \,|\, x \in \mathcal{M}_r\}| \geq k \}
    \nonumber
\end{equation}
For simplicity when $k=1$, we will neglect the subscript $k$ and write $\mathcal{M}$.
\label{def:data_set_margin}
\end{definition}

\begin{definition}[Feature Margin Indicator Function]
The feature margin indicator function is a function indicating whether a sample $x$ is in a features margin of interest, and is defined as follows: $M_r: \mathcal{X} \xrightarrow{} \{0\}\cup \mathcal{M}_r$
\begin{equation}
M_r(x_i) = 
\begin{cases}
x_{i_r}, & x_i\in \mathcal{M}_r \\
0, & \text{otherwise}
\nonumber
\end{cases}
\end{equation}
\end{definition}

\begin{definition}[Dataset Margin Indicator Function]
Given $k\in \mathbb{N}$, the dataset margin indicator function is an indicator function returning whether a sample is in the dataset margin.
\begin{equation}
M^{k}(x_i) = 
\begin{cases}
x_{i}, & \text{if} \quad x_i\in \mathcal{M}^{k} \\
0, & \text{otherwise} \\
\end{cases}
\label{eq_margin_indicator_function}
\end{equation}
As before, for $k=1$ we neglect the subscript $k$ and write $M$.
\label{def:dataset_margin_function}
\end{definition}

\subsection{Marginal Laplacian Score}
The Marginal Laplacian Score algorithm proposes to combine Laplacian Scores’ idea, of selecting features that preserve the local sample distances, and by that, the global structure, together with our main assumption that the most relevant data for anomaly detection lies at the margin of the features. By combining these ideas, MLS aims to find features that preserve the local margin's sample distance, meaning preserving the global structure of the margins.

Similar to the Laplacian Score, we aim to introduce a weighting parameter that will assign weight to the distances between marginal units. For that, we will use two different weighting terms, a sample-level weight $u_i$, and an interaction level weight $w_{ij}$.

The sample-level weight $u_i$ – indicates the significance of sample $i$ in defining the global margins. This weight should increase monotonically with the number of features sample $i$ is in their margin. Therefore, will be defined as in Eq. \ref{eq_unit_level_weight}.

\begin{equation}
u_i = log(\|M^k(x_i)\|_0 + 1)
\label{eq_unit_level_weight}
\end{equation}

Eq. \ref{eq_unit_level_weight} essentially counts the number of times a given sample is inside a feature's margin, indicating the importance of the sample in defining the global margin. The log function is used to gradually reduce the importance of being within another feature's margin. We note that any concave function can be used for this purpose. We further analyze and discuss the sample-level weight in section \ref{sample_weight_analysis}.

The interaction level weight $W_{ij}$ is similar to that of the Laplacian Score, it indicates how close samples $i$ and $j$ are in the global structure. However, unlike the Laplacian Score, MLS focuses solely on the margin's structure. Hence, it will be defined as in Eq. \ref{eq_interaction_level_weight}.

\begin{equation}
w_{ij} = e^{-\frac{\|M^k(x_i)-M^k(x_j)\|}{t}}
\label{eq_interaction_level_weight}
\end{equation}

Where $t$, a hyper parameter discussed later, determines the rate of decay. Eq. \ref{eq_interaction_level_weight}, calculates a decaying weight as a function of the distance of $x_i,x_j$ in the global margins space. It indicates the proximity of samples within the global margins space.

The final MLS score Eq. \ref{eq_mls_score} is a combination of the weights described in Eq. \ref{eq_unit_level_weight}, \ref{eq_interaction_level_weight} with the original LS Eq. \ref{eq_laplacian_score}. As with LS, a lower score indicates that the feature is more informative.

\begin{align}
MLS_r=
\sum_{ij\in \mathcal{M}^k}\frac{(f_{r_i}-f_{r_j})^2w_{ij}u_{i}}{Var(f_r)}
\label{eq_mls_score}
\end{align}

Under the assumption that the margins of the features will hold a substantial amount of the minority class or anomalies, the sum in Eq. \ref{eq_mls_score} and $w_{ij}$ will be more evenly affected by them, causing the structure learnt to hold information from both classes. Moreover, by considering a more even mix of majority and minority or anomaly samples, the interaction between them will have a stronger effect on the score. This is because the sum in Eq. \ref{eq_mls_score} will possess a more significant relative representation of pairs of samples, with one from each class. Therefore, allowing the selection of features better representing the difference between the classes or difference from normal samples to anomalies. 

Furthermore, by having MLS focus only on the margin of the features, correlated features have a lower effect on the score since only correlation that are related to the margins will have an effect. This allows MLS to more faithfully ignore correlated feature families and focus only on relevant information.

\subsection{A note on non-marginal data}\label{sec:non_marg}
Our marginal assumption applies to a wide array of data types, spanning from medical data to various engineering datasets. To enable MLS to contribute effectively across all dataset types, we propose employing any anomaly detection algorithm to select samples that will be regarded as the 'margins' of the dataset. This approach will help to ease our assumption regarding the marginalization of the minority class. We believe that this is a promising research direction that will be explored in the future and is beyond the scope of this paper.

\subsection{Matrix formulation}

As opposed to Eq. \ref{eq_laplacian_score}, the calculation is not symmetric between $i,j$ because $u_i$ is for a specific $i$. Nevertheless, a matrix formulation is desired. In Eq. \ref{eq_mls_matrix_formulation}, we reformulate MLS into matrix form, leading to an orders-of-magnitude acceleration in calculation speed.

\begin{observation}
$MLS_r$ in Eq. \ref{eq_mls_score} can be calculated using the martix form:
\begin{equation}
    MLS_r=\frac{f^TUDf+\mathds{1}^TUWf^2-2f^TWUf}{Var(f_r)}
    \label{eq_mls_matrix_formulation}
\end{equation}
\end{observation}
\begin{proof}
    The proof is in Appendix \ref{app:proof_mls_matrix}.
\end{proof}

\subsection{Temperature $t$ hyper parameter for large datasets}

In practical scenarios, the number of features could be quite extensive. This scenario can cause the numerator in Eq. 5 to become excessively large, resulting in a value that is too small, potentially leading to numerical errors in further calculations. The hyper parameter $t$ assists in addressing this issue. In many common normalization schemes like standard scaling, a majority of the feature samples tend to fall within the range of $[-1,1]$. In the case of samples within this range, the maximum distance in the numerator equals $2\sqrt{n_f}$ where $n_f$ represents the number of features. Thus, when $\sqrt{n_f} \gg t \rightarrow w_{ij} \ll 0$ causing numerical instability. To prevent such occurrences, we'll set $t$ to 10\% of the maximum value. Therefore, $t=\max(1,\frac{2\sqrt{n_f}}{10})$ will be used. The use of the $\max$ function ensures that $t$ does not become excessively small, particularly in smaller datasets.

\subsection{Differentiable Unsupervised Feature Selection with Marginal Laplacian Score}
DUFS complements LS by introducing stochastic gates to the input. It allows calculating LS over different subset of the feature space, striving to lower the overall score by closing or opening the stochastic gates on the input. By computing LS on each subset, DUFS focuses on the large structure formed by the informative features. As MLS is a modification of LS, our proposition involves replacing LS with MLS, thus creating DUFS-MLS. Utilizing MLS instead of LS enables DUFS to focus on features that better represent the marginal structures of the data. The integration of DUFS and MLS allows the latter to more effectively focus on genuine marginal features while disregarding irrelevant information and structure in the data. The formulation for DUFS-MLS is expressed by Eq. \ref{eq:dufs_mls}.

\begin{equation}
    L(\mu)=-\frac{\mathrm{Tr}[\frac{\Tilde{F}^TUD\Tilde{F}+\mathds{1}^TUW\Tilde{F}^2-2f^TWU\Tilde{F}}{Var(\Tilde{F})}]}{m\sum_{i=1}^d\mathds{P}(Z_i\geq0)+\delta}
    \label{eq:dufs_mls}
\end{equation}

Where $Var(\Tilde{F})$ is the variance of each feature. We note that dividing the variance may be counter intuitive since $\Tilde{F}$ is multiplied by a mask. In our preliminary experiments it was found useful and therefore it was kept.

\section{Experiments}
\subsection{Synthetic Data Experiment}

We conduct a series of experiments on synthetic datasets to comprehensively assess the efficacy of our proposed method, MLS. The Laplacian Score has two main disadvantages when applied to imbalanced data:
\begin{enumerate}
    \item Correlated Features - The Laplacian Score tends to selects the biggest correlated features families.
    \item Unknown Objective - The Laplacian Score selects the features that best agree with the whole sample graph, without a clear focus on relevant areas that are important for imbalanced datasets.
\end{enumerate}

The Laplacian Score assigns lower scores to features that agree well with the graph's structure, leading to the selection of correlated features. In imbalanced datasets, however, even highly correlated features may yield different results. Consequently,  these features families cannot be effectively filtered out in advance. As a result, the Laplacian score will tends to favor the selection of correlated features.

In addition, the Laplacian Score tends to select features that primarily capture the overall structure of the entire dataset, often influenced by the dominant majority class distribution. In imbalanced datasets, however, it becomes crucial to account for the behavior of minority samples, as they play a pivotal role. Due to their limited impact on the graph structure, features chosen by the Laplacian Score may not be optimal for effectively distinguishing between the majority and minority classes.

To demonstrate MLS's ability in identifying marginal features, we implement three different setups with imbalance ratios of 0.90, 0.95, and 0.97, each repeated 100 times, and then average the results. We note the imbalance ratio as $\rho$. We design this experiment such that each setup poses a more challenging task for feature selection algorithms to accurately identify the relevant features.

Synthetic dataset setups:
\begin{itemize}
    \item \textbf{Setup I: 5 Marginal and 5 Non-Marginal} - The dataset consists of 5 non-marginal features randomly generated from a normal distribution.The 5 marginal features consist of separately randomly generated negative and positive samples, with the positive samples positioned at the margin of the negative samples. This setup provides a straightforward demonstration of how each feature selection algorithm identifies the marginal features. 
    \item \textbf{Setup II: 5 Marginal and 5 Non-Marginal Correlated} - The dataset consists of 5 non-marginal features randomly generated from a normal distribution, aiming to preserve a Pearson correlation of 0.9 between them. The 5 marginal features are generated as before. This setup aims to demonstrate MLS's ability to ignore correlations that are not attributed to the margins of the features.
    \item \textbf{Setup III: 5 Marginal, 5 Non-Marginal Correlated, and 90 Non-Marginal} - The dataset consists of 5 non-marginal and 5 marginal features as before, along with an additional 90 randomly generated normally distributed features. This setup aims to demonstrate MLS's ability to disregard correlations not associated with the feature margins, as well as handle numerous non-informative features.
\end{itemize}

We present the results in Table \ref{table_1_classification_accuracy}. MLS consistently outperformed LS across all setups and three imbalance ratios. The LS method experienced a significant drop in performance with the introduction of correlated features, whereas MLS's performance was only slightly affected. Under our assumption that the minority class is more frequently present in the margins, correlated features which are not correlated in the margins as well are more likely to be irrelevant for the detection of the minority class. The lack of correlation in the margins implies that, even though these features may exhibit some degree of correlation within the overall dataset, this correlation does not contribute significantly to the detection of the minority class.

In other words, the correlation structure observed in these features is not indicative of the minority instances we seek to detect, if a correlation exists only in most of the data but not in the regions where minority instances are prevalent, it suggests that the correlated behavior is not a distinguishing factor for the minority class. Therefore, when selecting features in imbalance classification, it becomes crucial to not be biased to the overall correlation but also assess whether this correlation extends to the margins where minority instances are more likely to occur. Features that lack correlation in these critical regions may be deemed less relevant for capturing abnormal patterns and may not contribute significantly to the effectiveness of anomaly detection models. In light of that MLS's performance was only slightly affected and achieved a near perfect result, it highlights it's ability to focus on the margin of each feature, allowing it to select the most informative samples.


\begin{table}[t]
\begin{center}
\begin{small}
\begin{sc}
\begin{tabular}{@{}clllllllllllll@{}}
\toprule
\textbf{Experiment} & & \multicolumn{2}{c}{MLS} & \multicolumn{2}{c}{LS} \\
\cmidrule(lr){3-4}\cmidrule(lr){5-6}
&  & Mean & Std. & Mean & Std.\\
\midrule
Setup I & $\rho=0.9$ & \textbf{100} & 0 & 90.4 & 12.48\\
& $\rho=0.95$ & \textbf{100} & 0 & 52.8 & 17.78\\
& $\rho=0.97$ & \textbf{100} & 0 & 37.8 & 16.7\\
Setup II & $\rho=0.9$ & \textbf{100} & 0 & 0 & 0 \\
& $\rho=0.95$ & \textbf{99.8} & 1.98 & 0 & 0 \\
& $\rho=0.97$ & \textbf{98} & 8.71 & 0 & 0\\
Setup III & $\rho=0.9$ & \textbf{100} & 0 & 0 & 0\\
& $\rho=0.95$ & \textbf{99.8} & 1.98 & 0 & 0\\
& $\rho=0.97$ & \textbf{98} & 8.71 & 0 & 0\\
\bottomrule
\end{tabular}
\caption{A comparison of the predictive accuracy over synthetic datasets. The imbalance ratio of the dataset is noted as $\rho$.}
\label{table_1_classification_accuracy}
\end{sc}
\end{small}
\end{center}
\vskip -0.1in
\end{table}

\begin{table*}[ht]
\centering
\resizebox{\textwidth}{!}{
\begin{tabular}{@{}cllllllllllllllll@{}}
\toprule
 \textbf{Experiment}                         &  &   
\multicolumn{2}{c}{MLS} & 
\multicolumn{2}{c}{LS}  &
\multicolumn{2}{c}{DUFS-MLS} & \multicolumn{2}{c}{DUFS}
\\

\cmidrule(lr){3-4}\cmidrule(lr){5-6}\cmidrule(lr){7-8}\cmidrule{9-10}
 &  Setup & Mean & Std. &  Mean & Std. & Mean & Std. & Mean & Std.\\

\midrule
Unmodified &  $5$ features & 0.902 & 0.103 & 0.895 & 0.103 & \textbf{0.916} & 0.078 & 0.909 & 0.072\\
& $7$ features & 0.916 & 0.086 & 0.904 & 0.091 & \textbf{0.933} & 0.072 & 0.922 & 0.069\\
&$10$ features & 0.931 & 0.077 & 0.919 & 0.084 & \textbf{0.943} & 0.064 & 0.936 & 0.069\\
&$12$ features & \textbf{0.947} & 0.062 & 0.928 & 0.081 & \textbf{0.947} & 0.062 & 0.933 & 0.078\\

Noisy &  $5$ features & \textbf{0.850} & 0.152 & 0.726 & 0.216 & 0.834 & 0.134 & 0.576 & 0.112\\
& $7$ features & \textbf{0.854} & 0.140 & 0.726 & 0.221 & 0.821 & 0.134 & 0.629 & 0.139\\
&$10$ features & 0.858 & 0.143 & 0.735 & 0.216 & \textbf{0.867} & 0.123 & 0.660 & 0.143\\
&$12$ features & \textbf{0.866} & 0.132 & 0.848 & 0.140 & \textbf{0.866} & 0.109 & 0.654 & 0.138\\

\bottomrule
\end{tabular}
}
\caption{A comparison of the AUC ROC scores over public datasets.}
\label{tbl:experimentsPublicdatasets}
\end{table*}

\begin{table*}
\centering
\resizebox{\textwidth}{!}{
\begin{tabular}{@{}llllllllllllllll@{}}
\toprule
& \multicolumn{2}{c}{MLS} &
\multicolumn{2}{c}{LS} &
\multicolumn{2}{c}{DUFS-MLS} & \multicolumn{2}{c}{DUFS} \\
\cmidrule(lr){2-3}\cmidrule(lr){4-5}\cmidrule(lr){6-7}\cmidrule{8-9}
Setup & Mean & Std. & Mean & Std. & Mean & Std. & Mean & Std. \\
\midrule
$5$ features & \textbf{0.749} & 0.354 & 0.480 & 0.467 & 0.737 & 0.342 & 0.091 & 0.161 \\
$7$ features & \textbf{0.695} & 0.359 & 0.461 & 0.468 & 0.660 & 0.357 & 0.099 & 0.154 \\
$10$ features & \textbf{0.636} & 0.369 & 0.455 & 0.472 & 0.601 & 0.347 & 0.111 & 0.167 \\
$12$ features & \textbf{0.629} & 0.357 & 0.507 & 0.412 & 0.584 & 0.359 & 0.102 & 0.147 \\
\bottomrule
\end{tabular}
}
\caption{A comparison of the accuracy performance over public datasets in the noisy setup.}
\label{tbl:experimentsPublicdatasetsNoisyAccuracy}
\end{table*}

\begin{figure*}[ht]
\centering
{\includegraphics[width=\textwidth]{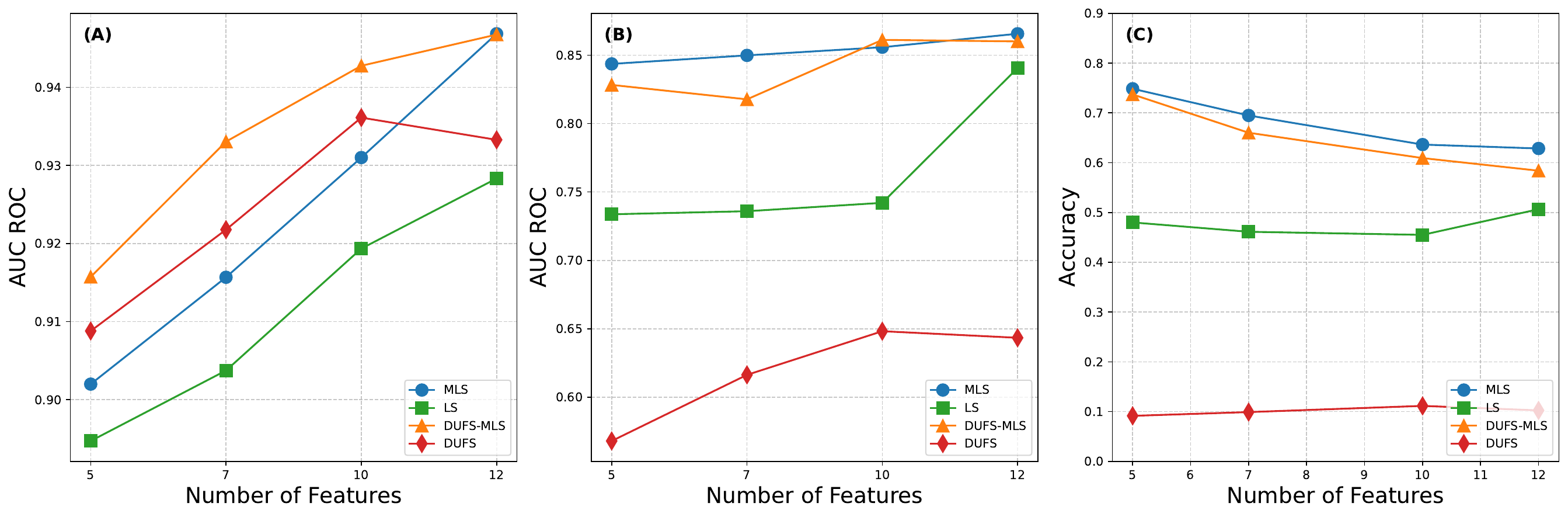}}
\caption{Comparison of unmodified and noisy data results. (A) Unmodified data setup results showing the AUC ROC values. (B) Noisy data setup results showing the AUC ROC values. (C) Predictive accuracy metrics for noisy data setup.}
\label{raw_noisy_auc}
\end{figure*}

\subsection{Public Data Experiment}

To further asses our proposed method MLS, we conduct a series of experiments on publicly  available datasets. We have selected a group of datasets from the Outlier Detection datasets (ODDS) library \cite{Rayana:2016}. Our focus has been on datasets where our marginal assumption, that anomalies are more likely to be in the margins of the features, holds true. Consequently, we excluded image-related anomaly datasets and retained engineering, medical, and other tabular datasets. A detailed table providing information on these datasets can be found in the appendix. We conducted experiments using two setups: unmodified and noisy. The unmodified setup involves using the datasets as is. Recognizing that these datasets are small and may not adequately represent scenarios requiring unsupervised feature selection methods, we introduced the noisy setup. 

Noisy setup consisted of adding 10 random correlated features with Pearson correlation coefficient of 0.9, generated from a multivariate normal distribution with a mean of 0 and a standard deviation of 1. Additionally, normally distributed random features with a mean of 0 and a standard deviation of 1 were incorporated until the dataset comprised 309 features. With an increased number of features in the noisy setup, the algorithm must distinguish between informative features, correlated random features, and those uncorrelated. This critical process of distinction is essential for accurate feature selection, making the noisy setup a more challenging and practical scenario. Therefore, it is the most important experiment to test the performance of different feature selection algorithms. Out of the 14 datasets chosen, 3 had too few features for evaluation in the unmodified setup, and they were considered only for the noisy set up.

Each experiment involved recording the outcomes of a feature selection method through a binary classification approach. We apply a stratified split for each dataset, creating training, validation, and holdout sets with the following percentages: 35\%, 35\%, and 30\%, respectively. We then used a feature selection algorithm to choose a subset of features. We utilize two classification methods, Logistic Regression (LR) \cite{cox1958regression} and Extra Trees (ET) \cite{geurts2006extremely}, to evaluate the predictive performance of feature subsets selected by the algorithms. We perform a random grid search with a budget of 10 for feature selection and 6 for the classification models to determine the best hyper parameters for both algorithms. For LS which had less than 10 hyper parameter combination, we use an exhaustive grid search. We select the optimal combination based on its performance score on the validation set, determined by the area under the receiver operating characteristic (AUC ROC) curve. Subsequently, we use and evaluate this chosen combination on a holdout set using the AUC ROC score. In the noisy setup, we also employ predictive accuracy, which represents the number of features selected from the non-added noisy features. For each dataset, we conduct the experiment 10 times, each with a different random split. The final results represent the mean of all 10 runs for each dataset and across all datasets.

The results of the unmodified and noisy setup are summarized and presented in  Table \ref{tbl:experimentsPublicdatasets}.
While conducting a performance comparison among MLS, LS, DUFS, and DUFS-MLS, we specifically focus on pairwise comparisons: MLS vs. LS and DUFS vs. DUFS-MLS. This approach ensures a comprehensive understanding of their respective strengths and weaknesses.
In the unmodified setup, the AUC ROC average results are as follows: MLS 0.924, LS 0.912, DUFS-MLS 0.935, and DUFS 0.925. Across experiments, one of the MLS modifications consistently outperforms the rest and that MLS and DUFS-MLS outperform their respective counterparts. In the noisy setup, a similar trend is observed. The AUC ROC average results are as follows: MLS 0.857, LS 0.759, DUFS-MLS 0.847, and DUFS 0.630. Across experiments, one of the MLS modifications consistently outperforms the rest. It can be noted that MLS and DUFS-MLS also outperform their respective counterparts in terms of AUC ROC results in the noisy setup.
LS and DUFS exhibit a notable performance deterioration in the noisy setup, which we attribute to the challenging nature of the noisy setup. In contrast, the performance of MLS modifications was only slightly affected, attesting to how MLS manages to focus on the margin of each feature.

The predictive accuracy results, specifically used in the noisy setup, are presented in Table \ref{tbl:experimentsPublicdatasetsNoisyAccuracy}, showcasing average scores of 0.677, 0.476, 0.646, and 0.101 respectively. Notably, MLS outperforms, closely followed by DUFS-MLS. The margin from the other algorithms highlights the performance gap.

\subsection{Sample-Level Weight Analysis}
\label{sample_weight_analysis}
We analyze the influence of the proposed sample-level weight, demonstrating its effectiveness and highlighting the importance of choosing an appropriate quantile.
We calculate sample-level weight distributions for negative and positive samples across different quantiles, ranging from 0.01 to 0.3, defining various feature margins. Subsequently, we computed Kolmogorov-Smirnov \cite{massey1951kolmogorov} distances between the distributions. Furthermore, we perform the Kolmogorov-Smirnov test with the distances to determine if the two distributions are statistically different. The average distances and corresponding p-values for each quantile across all datasets are jointly presented in Figure \ref{sample_level_weight_ks_pval}.

\begin{figure*}[htbp]
\vskip 0.2in
\centering
{\includegraphics[width=\textwidth]{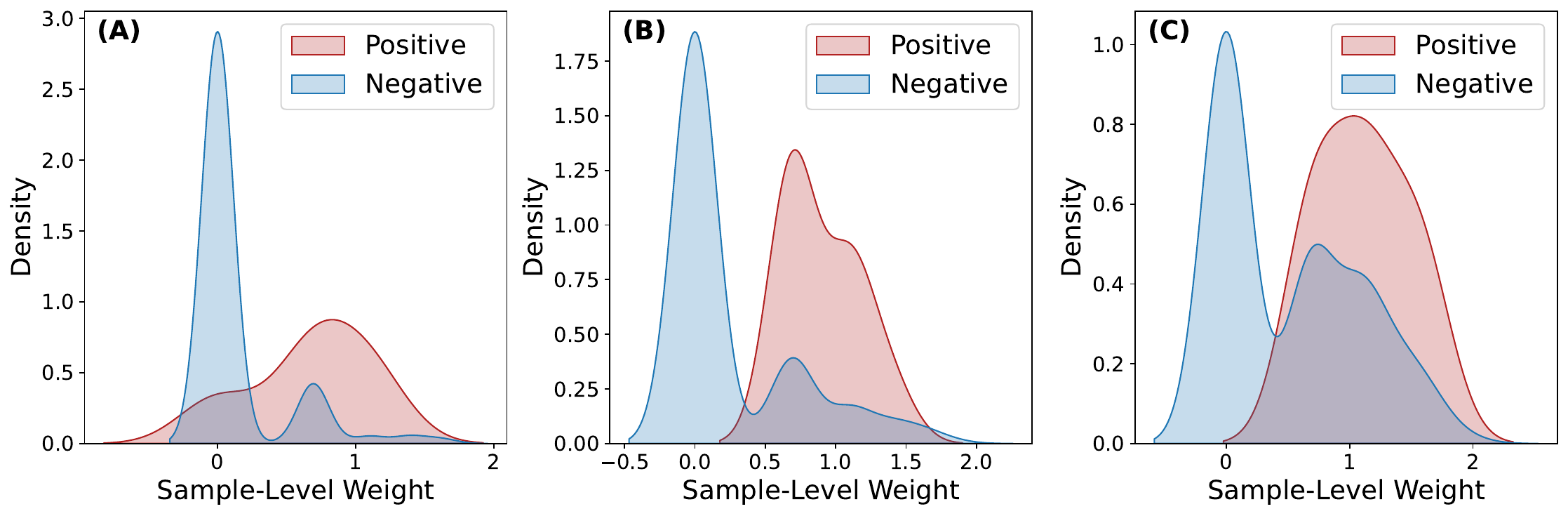}}
\caption{Sample-level weights across the glass dataset for various margin quantiles. In (A), (B), and (C), corresponding to quantiles 0.025, 0.05, and 0.1, respectively, the depicted figures showcase the distribution of weights at the sample level.}
\label{sample_level_weight_glass}
\vskip -0.2in
\end{figure*}

\begin{figure}[htbp]
\begin{center}
\centerline{\includegraphics[width=\columnwidth]{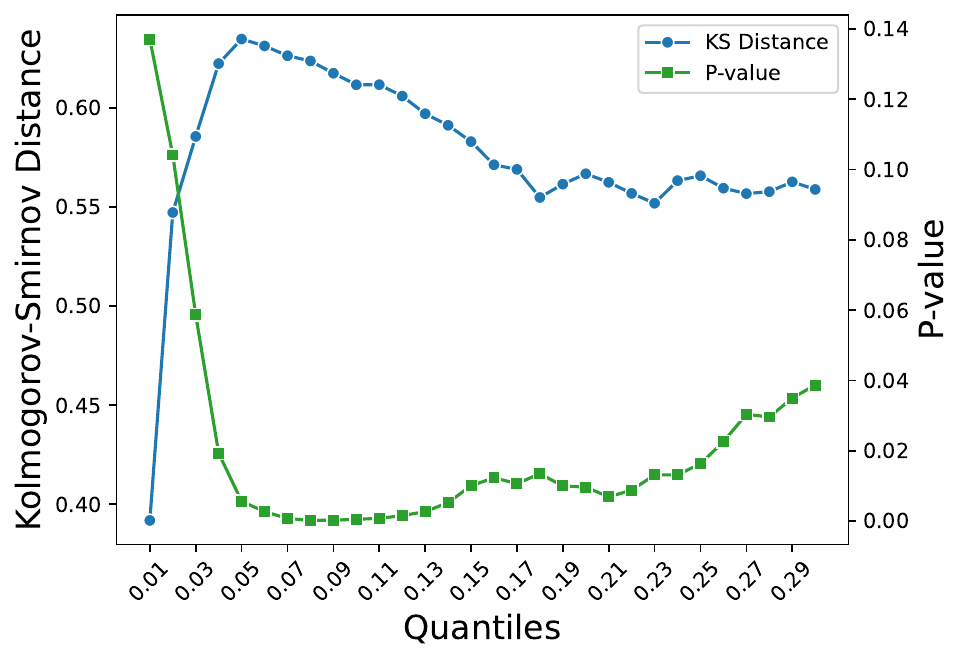}}
\caption{Mean Kolmogorov-Smirnov distances and corresponding p-values across various quantiles.}
\label{sample_level_weight_ks_pval}
\end{center}
\vskip -0.5in
\end{figure}

For the majority of quantiles, the p-values indicate statistical significance, signifying differences between the two distributions. However, it is important to note a distinction in the values of the distances, where a larger value is considered more favorable. Notably, under the assumption of marginality, the sample-level weight may function as an unsupervised separator between positive and negative samples. Fig. \ref{sample_level_weight_glass} provides a visual representation of sample-weight distributions across the glass dataset at three distinct quantiles. Notably, a noticeable distribution shift is evident between the positive and negative classes. The calculated Kolmogorov-Smirnov distances for these distributions are 0.617, 0.737, and 0.502 for the 0.025, 0.05 and 0.1 quantiles respectively. While all three distributions differ, the 0.05 quantile displays the most pronounced shift. This distinction is crucial as it aids the algorithm in discerning the most informative features for decision-making.

\subsection{Empirical Proof of the Marginal Assumption}

Our main assumption, asserting that the minority class will be more prevalent in the dataset margin as per Definition \ref{def:data_set_margin}, is crucial for focusing unsupervised feature selection algorithms on the specific task of distinguishing between the minority and majority classes. The sample-level weights serve as a metric, shedding light on how frequently a sample appears within the dataset margin. Consequently, the analysis in Section \ref{sample_weight_analysis} provides empirical validation for our assumption. In Fig. \ref{sample_level_weight_ks_pval}, the depiction of mean Kolmogorov-Smirnov distances between minority (positive) and majority (negative) classes across 11 datasets, spanning diverse domains like medical, engineering, chemistry, and physics, reinforces the idea that the minority class is more likely to be situated within the dataset's margin. Moreover, Fig. \ref{sample_level_weight_ks_pval} incorporates the p-value of the Kolmogorov-Smirnov distance, signifying its statistical significance and thereby upholding our confidence in the validity of our assumption across various domains.

\section{Conclusions}
We propose a feature selection algorithm called Marginal Laplacian Score (MLS), which enhances the effectiveness of the well-known Laplacian Score (LS) in addressing imbalanced datasets. Motivated by the concept that anomalous samples exhibit marginal behaviors, MLS is designed to identify features that better align with the structure of marginal data. By introducing two novel elements, sample-level weight and marginal interaction level weight, which allows generating the MLS equation which preserves the datasets marginal graph. This leads to improved and robust results on synthetic and public datasets, showcasing MLS's ability to identify features that hold information for distinguishing between the majority and minority classes. Additionally, we empirically proved our main assumption, that minority or anomalous samples are more likely to exists in the margins of the features, proving MLS's usefulness holds on real world datasets.

Furthermore, our exploration reveals that MLS, as a modification of LS, proves advantageous when incorporated into the Differentiable Unsupervised Feature Selection (DUFS) algorithm. This integration, resulting in DUFS-MLS, showcases notable improvements, emphasizing the potential of MLS in the feature selection framework. Improvements which achieve enhanced performance on public datasets.

Future research efforts can be directed towards generalizing the method to accommodate datasets with less marginal characteristics, as suggested in Section \ref{sec:non_marg}.
Another potential approach involves the development of hybrid methods that leverage MLS in or alongside other techniques to enhance feature selection outcomes.

\newpage
\section*{Acknowledgments}
We are grateful to Dr. Amitai Armon for his guidance through the process and review. His expertise significantly helped shape the final outcome.
\bibliography{paper}
\bibliographystyle{icml2024}

\newpage
\appendix
\onecolumn
\section{Appendix}
\subsection{Proof of MLS Matrix Form} \label{app:proof_mls_matrix}
\begin{proof} 
We define the matrix \(W\) as a square matrix with elements \(w_{ij}\) as follows:
\[
W = 
\begin{bmatrix}
w_{11} & w_{12} & \cdots & w_{1n} \\
w_{21} & w_{22} & \cdots & w_{2n} \\
\vdots & \vdots & \ddots & \vdots \\
w_{n1} & w_{n2} & \cdots & w_{nn}
\end{bmatrix}
\]

Next, we define the diagonal matrix \(D\) with the sum of each row of \(W\) on the diagonal, and zeros elsewhere as follows:
\[
D = 
\begin{bmatrix}
\sum_{j} w_{1j} & 0 & \cdots & 0 \\
0 & \sum_{j} w_{2j} & \cdots & 0 \\
\vdots & \vdots & \ddots & \vdots \\
0 & 0 & \cdots & \sum_{j} w_{nj}
\end{bmatrix}
\]
We define the matrix $U$ as the diagonal matrix with elements $s_{ij}$.\\
Now, let's consider the product of \(W\) and \(U\), denoted as \(WU\):
\[
WU
=
\begin{bmatrix}
w_{11}u_{1} & w_{12}u_{2} & \cdots & w_{1n}u_{n} \\
w_{21}u_{1} & w_{22}u_{2} & \cdots & w_{2n}u_{n} \\
\vdots & \vdots & \ddots & \vdots \\
w_{n1}u_{1} & w_{n2}u_{2} & \cdots & w_{nn}u_{n}
\end{bmatrix}
\]

Similarly, we obtain \(UW\).

Finally, \(UD\):
\[
UD =
\begin{bmatrix}
u_{1}\sum_{j} w_{1j} & 0 & \cdots & 0 \\
0 & u_{2}\sum_{j} w_{2j} & \cdots & 0 \\
\vdots & \vdots & \ddots & \vdots \\
0 & 0 & \cdots & u_{n}\sum_{j} w_{nj}
\end{bmatrix}
\]

Now that we have all the elements, we can calculate each element in Eq. \ref{eq_mls_matrix_formulation}:

\begin{align}
f^TWUf= & f^T\begin{bmatrix}
w_{11}u_{1} & w_{12}u_{2} & \cdots & w_{1n}u_{n} \\
w_{21}u_{1} & w_{22}u_{2} & \cdots & w_{2n}u_{n} \\
\vdots & \vdots & \ddots & \vdots \\
w_{n1}u_{1} & w_{n2}u_{2} & \cdots & w_{nn}u_{n}
\end{bmatrix}f=\nonumber \\ & \begin{bmatrix}
\sum_j f_j w_{j1}u_{1} & \sum_j f_j w_{j2}u_{2} & \cdots & \sum_j f_j w_{jn}u_{n}
\end{bmatrix}f = \sum_i \sum_j w_{ij} u_i f_j f_i
\label{eq_matrix_formulation_a}
\end{align}

\begin{align}
f^TUDf= & f^T\begin{bmatrix}
u_{1}\sum_{j} w_{1j} & 0 & \cdots & 0 \\
0 & u_{2}\sum_{j} w_{2j} & \cdots & 0 \\
\vdots & \vdots & \ddots & \vdots \\
0 & 0 & \cdots & u_{n}\sum_{j} w_{nj}
\end{bmatrix}
f=\nonumber \\ & \begin{bmatrix}
f_1 u_{1}\sum_{j} w_{1j} & f_2 u_{2}\sum_{j} w_{2j} & \cdots & f_n u_{n}\sum_{j} w_{nj}
\end{bmatrix}f = \sum_i \sum_j w_{ij} u_i f_i^2
\label{eq_matrix_formulation_b}
\end{align}

\begin{align}
\mathds{1}^TUWf^2= & \mathds{1}^T \begin{bmatrix}
w_{11}u_{1} & w_{12}u_{1} & \cdots & w_{1n}u_{1} \\
w_{21}u_{2} & w_{22}u_{2} & \cdots & w_{2n}u_{2} \\
\vdots & \vdots & \ddots & \vdots \\
w_{n1}u_{n} & w_{n2}u_{n} & \cdots & w_{nn}u_{n}
\end{bmatrix} f^2=\nonumber \\ & \mathds{1}^T \begin{bmatrix}
u_{1}\sum_{j} w_{1j} f_j^2 & u_{2}\sum_{j} w_{2j} f_j^2 & \cdots & u_{n}\sum_{j} w_{nj} f_j^2
\end{bmatrix} = \sum_i \sum_j w_{ij} u_i f_j^2
\label{eq_matrix_formulation_c}
\end{align}

Finally, by combining Eq. \ref{eq_matrix_formulation_a}, \ref{eq_matrix_formulation_b}, and \ref{eq_matrix_formulation_c}, we obtain:

\begin{align}
MLS_r= & \sum_{ij\in \mathcal{M}}\frac{(f_{r_i}-f_{r_j})^2w_{ij}u_{i}}{Var(f_r)}=\sum_{ij\in \mathcal{M}}\frac{(f_{r_i}^2+f_{r_j}^2-2f_{r_j}f_{r_i})^2w_{ij}u_{i}}{Var(f_r)}= \nonumber \\ & \sum_{ij\in \mathcal{M}}\frac{f_{r_i}^2w_{ij}u_{i}}{Var(f_r)}+\sum_{ij\in \mathcal{M}}\frac{f_{r_j}^2w_{ij}u_{i}}{Var(f_r)}-\sum_{ij\in \mathcal{M}}\frac{2f_{r_i}f_{r_j}w_{ij}u_{i}}{Var(f_r)} = \nonumber \\ &
\frac{f^TUDf+\mathds{1}^TUWf^2-2f^TWUf}{Var(f_r)}
\end{align}
\end{proof}

\subsection{Marginal Feature Illustration}
\label{app:marginal_feature_illustration}
\begin{figure}[htbp]
\begin{center}
\centerline{\includegraphics[width=0.5\columnwidth]{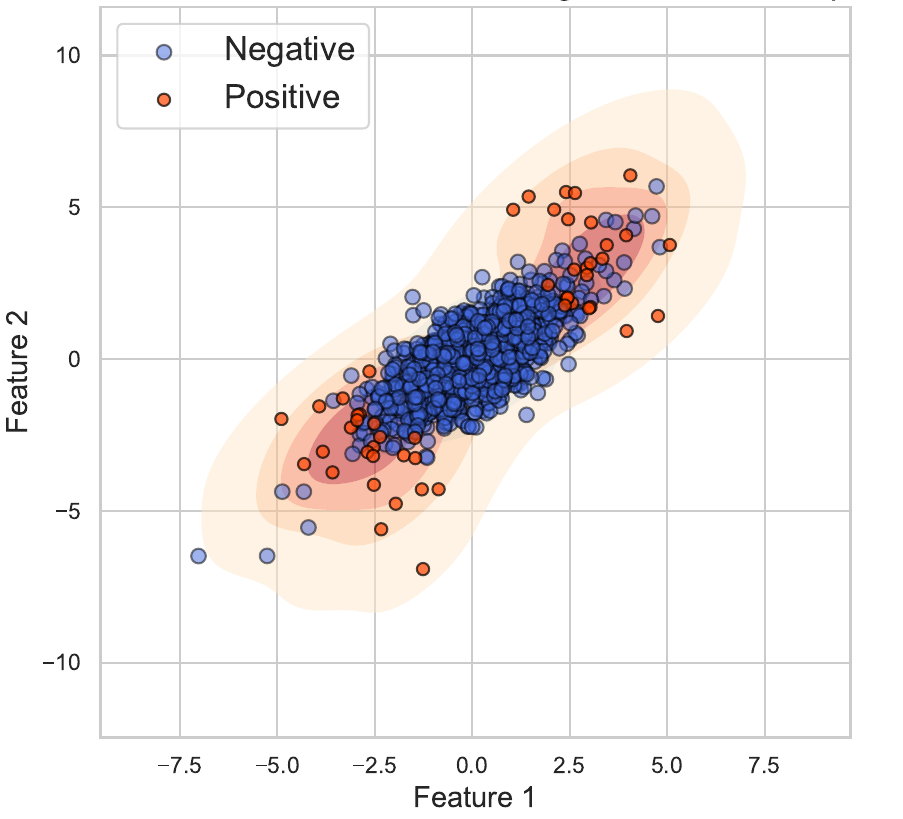}}
\caption{Theoretical illustration of a two-feature distribution with marginal positive samples.}
\label{theortical_kde_plot}
\end{center}
\end{figure}

\newpage 

\subsection{Public Dataset Description}
\begin{table}[h]
\centering
\begin{tabular}{lccc}
\hline
\textbf{Dataset} & \textbf{\# Points} & \textbf{\# Dimensions} & \textbf{\# Outliers (\%)} \\
\hline
Ionosphere & 351 & 33 & 126 (35.90\%) \\
Landsat & 6435 & 36 & 1333 (20.71\%) \\
Lympho & 148 & 18 & 6 (4.05\%) \\
Musk & 3062 & 166 & 97 (3.17\%) \\
Pendigits & 6870 & 16 & 156 (2.27\%) \\
Pima & 768 & 8 & 268 (34.90\%) \\
Satellite & 6435 & 36 & 2036 (31.64\%) \\
Thyroid & 3772 & 6 & 93 (2.47\%) \\
Vertebral & 240 & 6 & 30 (12.50\%) \\
Vowels & 1456 & 12 & 50 (3.43\%) \\
Waveform & 3443 & 21 & 100 (2.90\%) \\
WBC & 278 & 30 & 21 (5.56\%) \\
WDBC & 367 & 30 & 10 (2.72\%) \\
Wine & 129 & 13 & 10 (7.75\%) \\
\hline
\end{tabular}
\caption{Dataset Information}
\label{tab:dataset_info}
\end{table}

\subsection{Individual Public Dataset Scores - Raw Setup}

\begin{table}[htbp]
  \centering
  \begin{tabular}{lcccc}
    \toprule
    Dataset      & MLS      & LS       & DUFS-MLS  & DUFS     \\
    \midrule
    Ionosphere   & \textbf{0.931} & 0.841 & 0.897  & 0.880 \\
    Landsat      & 0.811 & \textbf{0.813} & 0.808  & 0.802 \\
    Lympho       & 0.690 & 0.678 & 0.821  & \textbf{0.870} \\
    Musk         & 0.999 & 0.987 & \textbf{1.000}  & 0.994 \\
    Pendigits    & 0.976 & \textbf{0.986} & 0.972  & 0.968 \\
    Satellite    & \textbf{0.855} & 0.864 & 0.851  & 0.837 \\
    Vowels       & 0.900 & \textbf{0.934} & 0.927  & 0.899 \\
    Waveform     & 0.791 & 0.804 & \textbf{0.827}  & 0.822 \\
    WBC          & \textbf{0.974} & 0.963 & \textbf{0.974}  & 0.952 \\
    WDBC         & 0.999 & \textbf{1.000} & 0.997  & 0.999 \\
    Wine         & \textbf{0.999} & 0.996 & \textbf{0.999}  & 0.973 \\
    \bottomrule
  \end{tabular}
  \caption{Raw 5 feature ROC-AUC scores}
  \label{tab:raw_5}
\end{table}

\begin{table}[htbp]
  \centering
  \begin{tabular}{lcccc}
    \toprule
    Dataset      & MLS      & LS       & DUFS-MLS  & DUFS     \\
    \midrule
    Ionosphere   & \textbf{0.949} & 0.857 & 0.893  & 0.909 \\
    Landsat      & \textbf{0.821} & 0.812 & 0.799  & 0.811 \\
    Lympho       & 0.762 & 0.727 & \textbf{0.953}  & 0.845 \\
    Musk         & 0.996 & 0.994 & 0.998  & \textbf{0.999} \\
    Pendigits    & 0.985 & \textbf{0.987} & 0.983  & 0.986 \\
    Satellite    & 0.856 & \textbf{0.865} & 0.839  & 0.856 \\
    Vowels       & 0.904 & 0.930 & \textbf{0.969}  & 0.913 \\
    Waveform     & 0.829 & 0.829 & 0.863  & \textbf{0.870} \\
    WBC          & \textbf{0.975} & 0.953 & 0.968  & 0.966 \\
    WDBC         & \textbf{1.000} & 0.999 & \textbf{1.000}  & 0.998 \\
    Wine         & 0.997 & 0.987 & \textbf{0.999}  & 0.986 \\
    \bottomrule
  \end{tabular}
  \caption{Raw 7 feature ROC-AUC scores}
  \label{tab:raw_7}
\end{table}

\begin{table}[htbp]
  \centering
  \begin{tabular}{lcccc}
    \toprule
    Dataset      & MLS      & LS       & DUFS-MLS  & DUFS     \\
    \midrule
    Ionosphere   & \textbf{0.953} & 0.859 & 0.925 & 0.913 \\
    Landsat      & \textbf{0.817} & 0.808 & 0.816 & 0.810 \\
    Lympho       & 0.793 & 0.774 & \textbf{0.903} & 0.891 \\
    Musk         & 0.999 & 0.990 & 0.999 & \textbf{1.000} \\
    Pendigits    & 0.991 & 0.993 & \textbf{0.994} & 0.993 \\
    Satellite    & 0.862 & 0.862 & \textbf{0.866} & 0.848 \\
    Vowels       & 0.968 & 0.979 & \textbf{0.990} & 0.989 \\
    Waveform     & 0.887 & 0.884 & \textbf{0.905} & 0.887 \\
    WBC          & \textbf{0.976} & 0.966 & 0.972 & 0.967 \\
    WDBC         & \textbf{1.000} & 0.998 & \textbf{1.000} & \textbf{1.000} \\
    Wine         & 0.995 & 0.999 & 0.998 & \textbf{1.000} \\
    \bottomrule
  \end{tabular}
  \caption{Raw 10 feature ROC-AUC scores}
  \label{tab:raw_10}
\end{table}

\begin{table}[htbp]
  \centering
  \begin{tabular}{lcccc}
    \toprule
    Dataset      & MLS      & LS       & DUFS-MLS  & DUFS     \\
    \midrule
    Ionosphere   & \textbf{0.955} & 0.872 & 0.919 & 0.915 \\
    Landsat      & \textbf{0.822} & 0.811 & 0.819 & 0.810 \\
    Lympho       & 0.909 & 0.790 & \textbf{0.942} & 0.797 \\
    Musk         & 0.997 & 0.995 & \textbf{1.000} & \textbf{1.000} \\
    Pendigits    & 0.996 & \textbf{0.997} & 0.996 & 0.994 \\
    Satellite    & 0.869 & 0.858 & 0.867 & \textbf{0.874} \\
    Vowels       & \textbf{0.997} & \textbf{0.997} & \textbf{0.997} & \textbf{0.997} \\
    Waveform     & 0.898 & \textbf{0.929} & 0.911 & 0.907 \\
    WBC          & 0.976 & 0.967 & 0.970 & \textbf{0.980} \\
    WDBC         & \textbf{1.000} & 0.999 & \textbf{1.000} & \textbf{1.000} \\
    Wine         & 0.996 & \textbf{0.997} & 0.993 & 0.994 \\
    \bottomrule
  \end{tabular}
  \caption{Raw 12 feature ROC-AUC scores}
  \label{tab:raw_12}
\end{table}

\newpage
\subsection{Individual Public Dataset Scores - Noisy Setup}
\begin{table}[htbp]
  \centering
  \begin{tabular}{lcccc}
    \toprule
    Dataset                & MLS      & LS       & DUFS-MLS  & DUFS     \\
    \midrule
    Ionosphere      & \textbf{0.941}    & 0.838    & 0.921    & 0.596    \\
    Landsat         & 0.787    & \textbf{0.817}    & 0.785    & 0.581    \\
    Lympho          & 0.669    & 0.594    & \textbf{0.688}    & 0.541    \\
    Musk            & \textbf{1.000}    & 0.999    & \textbf{1.000}    & 0.898    \\
    Pendigits       & \textbf{0.980}    & 0.465    & 0.922    & 0.558    \\
    Pima            & \textbf{0.715}    & 0.493    & 0.650    & 0.523    \\
    Satellite       & 0.821    & \textbf{0.866}    & 0.805    & 0.566    \\
    Thyroid         & \textbf{0.996}    & 0.990    & 0.993    & 0.517    \\
    Vertebral       & \textbf{0.808}    & 0.519    & 0.739    & 0.546    \\
    Vowels          & \textbf{0.728}    & 0.536    & 0.622    & 0.532    \\
    Waveform        & 0.518    & 0.534    & \textbf{0.845}    & 0.523    \\
    WBC             & 0.963    & \textbf{0.970}    & 0.963    & 0.503    \\
    WDBC            & 0.998    & \textbf{1.000}    & \textbf{1.000}    & 0.449    \\
    Wine            & \textbf{0.974}    & 0.545    & 0.747    & 0.730    \\
    \bottomrule
  \end{tabular}
  \caption{Noisy 5 feature ROC-AUC scores}
  \label{tab:noisy_5}
\end{table}

\begin{table}[htbp]
  \centering
  \begin{tabular}{lcccc}
    \toprule
    Dataset                & MLS      & LS       & DUFS-MLS  & DUFS     \\
    \midrule
    Ionosphere      & \textbf{0.924}    & 0.846    & 0.907    & 0.651    \\
    Landsat         & 0.787    & \textbf{0.811}    & 0.795    & 0.553    \\
    Lympho          & \textbf{0.701}    & 0.666    & 0.676    & 0.524    \\
    Musk            & 0.980    & 0.999    & \textbf{1.000}    & 0.982    \\
    Pendigits       & \textbf{0.987}    & 0.462    & 0.869    & 0.597    \\
    Pima            & \textbf{0.721}    & 0.490    & 0.679    & 0.513    \\
    Satellite       & 0.829    & \textbf{0.862}    & 0.808    & 0.671    \\
    Thyroid         & \textbf{0.998}    & 0.983    & 0.979    & 0.567    \\
    Vertebral       & \textbf{0.815}    & 0.606    & 0.717    & 0.507    \\
    Vowels          & \textbf{0.781}    & 0.535    & 0.616    & 0.493    \\
    Waveform        & 0.537    & 0.514    & \textbf{0.810}    & 0.586    \\
    WBC             & \textbf{0.971}    & 0.970    & 0.968    & 0.721    \\
    WDBC            & 0.999    & \textbf{1.000}    & 0.995    & 0.842    \\
    Wine            & \textbf{0.931}    & 0.419    & 0.671    & 0.606    \\
    \bottomrule
  \end{tabular}
  \caption{Noisy 7 feature ROC-AUC scores}
  \label{tab:noisy_7}
\end{table}

\begin{table}[htbp]
  \centering
  \begin{tabular}{lcccc}
    \toprule
    Dataset                & MLS      & LS       & DUFS-MLS  & DUFS     \\
    \midrule
    Ionosphere      & 0.885    & 0.870    & \textbf{0.935}    & 0.651    \\
    Landsat         & 0.815    & \textbf{0.819}    & 0.811    & 0.553    \\
    Lympho          & \textbf{0.706}    & 0.660    & 0.649    & 0.524    \\
    Musk            & \textbf{1.000}    & 0.999    & \textbf{1.000}    & 0.982    \\
    Pendigits       & \textbf{0.987}    & 0.444    & 0.972    & 0.597    \\
    Pima            & 0.736    & 0.492    & \textbf{0.748}    & 0.513    \\
    Satellite       & 0.860    & \textbf{0.870}    & 0.856    & 0.671    \\
    Thyroid         & \textbf{0.998}    & 0.988    & 0.995    & 0.567    \\
    Vertebral       & \textbf{0.806}    & 0.591    & 0.741    & 0.507    \\
    Vowels          & \textbf{0.762}    & 0.554    & 0.699    & 0.493    \\
    Waveform        & 0.520    & 0.512    & \textbf{0.815}    & 0.586    \\
    WBC             & 0.972    & \textbf{0.975}    & 0.971    & 0.722    \\
    WDBC            & \textbf{1.000}    & \textbf{1.000}    & \textbf{1.000}    & 0.842    \\
    Wine            & \textbf{0.959}    & 0.518    & 0.946    & 0.606    \\
    \bottomrule
  \end{tabular}
  \caption{Noisy 10 feature ROC-AUC scores}
  \label{tab:noisy_10}
\end{table}

\begin{table}[htbp]
  \centering
  \begin{tabular}{lcccc}
    \toprule
    Dataset                & MLS      & LS       & DUFS-MLS  & DUFS     \\
    \midrule
    Ionosphere      & 0.901    & 0.902    & \textbf{0.940}    & 0.631    \\
    Landsat         & 0.822    & 0.815    & \textbf{0.825}    & 0.686    \\
    Lympho          & \textbf{0.752}    & 0.589    & 0.692    & 0.505    \\
    Musk            & \textbf{1.000}    & 0.999    & \textbf{1.000}    & 0.998    \\
    Pendigits       & \textbf{0.990}    & 0.970    & 0.936    & 0.684    \\
    Pima            & \textbf{0.729}    & 0.618    & 0.711    & 0.572    \\
    Satellite       & 0.866    & \textbf{0.870}    & 0.861    & 0.673    \\
    Thyroid         & \textbf{0.997}    & 0.995    & 0.993    & 0.541    \\
    Vertebral       & \textbf{0.829}    & 0.816    & 0.801    & 0.565    \\
    Vowels          & 0.802    & \textbf{0.849}    & 0.732    & 0.532    \\
    Waveform        & 0.542    & 0.677    & \textbf{0.838}    & 0.607    \\
    WBC             & 0.967    & \textbf{0.973}    & \textbf{0.973}    & 0.751    \\
    WDBC            & \textbf{1.000}    & \textbf{1.000}    & \textbf{1.000}    & 0.859    \\
    Wine            & \textbf{0.931}    & 0.806    & 0.816    & 0.556    \\
    \bottomrule
  \end{tabular}
  \caption{Noisy 12 feature ROC-AUC scores}
  \label{tab:noisy_12}
\end{table}
\subsection{Disclaimer}
https://edc.intel.com/content/www/us/en/products/performance/benchmarks/overview/

\end{document}